\documentclass[12pt]{article}

\usepackage{amsmath}
\usepackage{amsthm}
\usepackage{amsfonts}
\usepackage[dvipsnames]{xcolor}
\usepackage{graphicx}
\usepackage{multirow}
\usepackage{hyperref}
\usepackage{algorithm}
\usepackage{algpseudocode}
\usepackage{listings}
\usepackage{setspace}
\usepackage[left=3cm,top=3cm,right=3cm,bottom=3cm]{geometry}
\usepackage{breakcites}

\lstdefinestyle{custompython}{
  belowcaptionskip=1\baselineskip,
  breaklines=true,
  frame=L,
  language=Python,
  showstringspaces=false,
  basicstyle=\footnotesize\ttfamily,
  keywordstyle=\bfseries\color{blue!60!black},
  identifierstyle=\color{black},
  stringstyle=\color{red!50!black},
}
\newtheorem{lemma}{Lemma}
\newtheorem{proposition}{Proposition}
\theoremstyle{definition}
\newtheorem{definition}{Definition}

\title{Formalized Conceptual Spaces with a Geometric Representation of Correlations}
\author{Lucas Bechberger\thanks{Corresponding author, ORCID: 0000-0002-1962-1777}$\;$ and Kai-Uwe K\"uhnberger\\Institute of Cognitive Science, Osnabr\"uck University\\ \emph{\{lucas.bechberger, kai-uwe.kuehnberger\}@uni-osnabrueck.de}}
\date{}
\begin{document}
\onehalfspacing
\maketitle

\begin{abstract}
The highly influential framework of conceptual spaces provides a geometric way of representing knowledge. Instances are represented by points in a similarity space and concepts are represented by convex regions in this space. After pointing out a problem with the convexity requirement, we propose a formalization of conceptual spaces based on fuzzy star-shaped sets. Our formalization uses a parametric definition of concepts and extends the original framework by adding means to represent correlations between different domains in a geometric way. Moreover, we define various operations for our formalization, both for creating new concepts from old ones and for measuring relations between concepts. We present an illustrative toy-example and sketch a research project on concept formation that is based on both our formalization and its implementation.
\end{abstract}

\section{Introduction}

One common criticism of symbolic AI approaches is that the symbols they operate on do not contain any meaning: For the system, they are just arbitrary tokens that can be manipulated in some way. This lack of inherent meaning in abstract symbols is called the “symbol grounding problem” \cite{Harnad1990}. One approach towards solving this problem is to devise a grounding mechanism that connects abstract symbols to the real world, i.e., to perception and action.\\

The framework of conceptual spaces \cite{Gardenfors2000,Gardenfors2014} attempts to bridge this gap between symbolic and subsymbolic AI by proposing an intermediate conceptual layer based on geometric representations.
A conceptual space is a similarity space spanned by a number of quality dimensions that are based on perception and/or subsymbolic processing. Convex regions in this space correspond to concepts. Abstract symbols can thus be grounded in reality by linking them to regions in a conceptual space whose dimensions are based on perception.

The framework of conceptual spaces has been highly influential in the last 15 years within cognitive science and cognitive linguistics \cite{Douven2011,Warglien2012,Fiorini2013,Zenker2015}. It has also sparked considerable research in various subfields of artificial intelligence, ranging from robotics and computer vision \cite{Chella2001,Chella2003,Chella2005} over the semantic web and ontology integration \cite{Dietze2008,Adams2009a} to plausible reasoning \cite{Schockaert2011,Derrac2015}.\\

One important question is however left unaddressed by these research efforts: How can an (artificial) agent learn about meaningful regions in a conceptual space purely from unlabeled perceptual data?
Our overall approach for solving this concept formation problem is to devise an incremental clustering algorithm that groups a stream of unlabeled observations (represented as points in a conceptual space) into meaningful regions. In this paper, we lay the foundation for this approach by developing a thorough formalization of the conceptual spaces framework.\footnote{This chapter is a revised and extended version of research presented in the following workshop and conference papers: \cite{Bechberger2017AIC, Bechberger2017KI, Bechberger2017SGAI, Bechberger2017NeSy}}

In Section \ref{CS}, we point out that G\"ardenfors' convexity requirement prevents a geometric representation of correlations. We resolve this problem by using star-shaped instead of convex sets. Our mathematical formalization presented in Section \ref{FSSSS} defines concepts in a parametric way that is easily implementable. We furthermore define various operations in Section \ref{Operations}, both for creating new concepts from old ones and for measuring relations between concepts. Moreover, in Section \ref{Implementation} we describe our implementation of the proposed formalization and illustrate it with a simple conceptual space for fruits. In Section \ref{RelatedWork}, we summarize other existing formalizations of the conceptual spaces framework and compare them to our proposal. Finally, in Section \ref{Outlook} we give an outlook on future work with respect to concept formation before concluding the paper in Section \ref{Conclusion}.
\section{Conceptual Spaces}
\label{CS}

This section presents the cognitive framework of conceptual spaces as described in \cite{Gardenfors2000} and introduces our formalization of dimensions, domains, and distances. Moreover, we argue that concepts should be represented by star-shaped sets instead of convex sets.

\subsection{Definition of Conceptual Spaces}
\label{CS:Definition}

A conceptual space is a similarity space spanned by a set $D$ of so-called ``quality dimensions''. Each of these dimensions $d \in D$ represents a cognitively meaningful way in which two stimuli can be judged to be similar or different. Examples for quality dimensions include temperature, weight, time, pitch, and hue. We denote the distance between two points $x$ and $y$ with respect to a dimension $d$ as $|x_d - y_d|$.

A domain $\delta \subseteq D$ is a set of dimensions that inherently belong together. Different perceptual modalities (like color, shape, or taste) are represented by different domains. The color domain for instance consists of the three dimensions hue, saturation, and brightness.
G\"{a}rdenfors argues based on psychological evidence \cite{Attneave1950,Shepard1964} that distance within a domain $\delta$ should be measured by the weighted Euclidean metric $d_E$: 
$$d_E^{\delta}\left(x,y, W_{\delta}\right) = \sqrt{\sum_{d \in \delta} w_{d} \cdot | x_{d} - y_{d} |^2}$$
The parameter $W_{\delta}$ contains positive weights $w_{d}$ for all dimensions $d \in \delta$ representing their relative importance. We assume that $\textstyle\sum_{d \in \delta} w_{d} = 1$.\\


The overall conceptual space $CS$ can be defined as the product space of all dimensions. Again, based on psychological evidence \cite{Attneave1950,Shepard1964}, G\"{a}rdenfors argues that distance within the overall conceptual space should be measured by the weighted Manhattan metric $d_M$ of the intra-domain distances. Let $\Delta$ be the set of all domains in $CS$. We define the overall distance within a conceptual space as follows:
$$
d_C^{\Delta}\left(x,y,W\right) := \sum_{\delta \in \Delta} w_{\delta} \cdot d_E^{\delta}\left(x,y,W_{\delta}\right)
= \sum_{\delta \in \Delta}w_{\delta} \cdot \sqrt{\sum_{d \in \delta} w_{d} \cdot |x_{d} - y_{d}|^2}
$$
The parameter $W = \left\langle W_{\Delta},\left\{W_{\delta}\right\}_{\delta \in \Delta}\right\rangle$ contains $W_{\Delta}$, the set of positive domain weights $w_{\delta}$. We require that $\textstyle\sum_{\delta \in \Delta} w_{\delta} = |\Delta|$. Moreover, $W$ contains for each domain $\delta \in \Delta$ a set $W_{\delta}$ of dimension weights as defined above. The weights in $W$ are not globally constant, but depend on the current context. One can easily show that $d_C^{\Delta}\left(x,y,W\right)$ with a given $W$ is a metric.\\

The similarity of two points in a conceptual space is inversely related to their distance. G\"{a}rdenfors expresses this as follows :
$$Sim\left(x,y\right) = e^{-c \cdot d\left(x,y\right)}\quad \text{with a constant}\; c >0 \; \text{and a given metric}\; d$$

Betweenness is a logical predicate $B\left(x,y,z\right)$ that is true if and only if $y$ is considered to be between $x$ and $z$. It can be defined based on a given metric $d$: 
$$B_d\left(x,y,z\right) :\iff d\left(x,y\right) + d\left(y,z\right) = d\left(x,z\right)$$

The betweenness relation based on the Euclidean metric $d_E$ results in the straight line segment connecting the points $x$ and $z$, whereas the betweenness relation based on the Manhattan metric $d_M$ results in an axis-parallel cuboid between the points $x$ and $z$.
We can define convexity and star-shapedness based on the notion of betweenness:

\begin{definition}
\label{def:Convexity}
(Convexity)\\
A set $C \subseteq CS$ is \emph{convex} under a metric $d \;:\iff$

\hspace{1cm}$\forall {x \in C, z \in C, y \in CS}: \left(B_d\left(x,y,z\right) \rightarrow y \in C\right)$
\end{definition}

\begin{definition}
\label{def:StarShapedSet}
(Star-shapedness)\\
A set $S \subseteq CS$ is \emph{star-shaped} under a metric $d$ with respect to a set $P \subseteq S \;:\iff$ 

\hspace{1cm}$\forall {p \in P, z \in S, y \in CS}: \left(B_d\left(p,y,z\right) \rightarrow y \in S\right)$
\end{definition}

Convexity under the Euclidean metric $d_E$ is equivalent to the common definition of convexity in Euclidean spaces. The only sets that are considered convex under the Manhattan metric $d_M$ are axis-parallel cuboids.\\

G\"{a}rdenfors distinguishes properties like ``red'', ``round'', and ``sweet'' from full-fleshed concepts like ``apple'' or ``dog'' by observing that properties can be defined on individual domains (e.g., color, shape, taste), whereas full-fleshed concepts involve multiple domains.

\begin{definition}
\label{def:CriterionP}
(Property) \cite{Gardenfors2000}\\
A \emph{natural property} is a convex region of a domain in a conceptual space.
\end{definition}

Full-fleshed concepts can be expressed as a combination of properties from different domains. These domains might have a different importance for the concept which is reflected by so-called ``salience weights''. Another important aspect of concepts are the correlations between the different domains \cite{Medin1988}, which are important for both learning \cite{Billman1996} and reasoning \cite[Chapter 8]{Murphy2002}.

\begin{definition}
\label{def:CriterionC}
(Concept) \cite{Gardenfors2000}\\
A \emph{natural concept} is represented as a set of convex regions in a number of domains together with an assignment of salience weights to the domains and information about how the regions in different domains are correlated.
\end{definition}

\subsection{An Argument Against Convexity}
\label{CS:ArgumentAgainstConvexity}

\cite{Gardenfors2000} does not propose any concrete way for representing correlations between domains. As the main idea of the conceptual spaces framework is to find a geometric representation of conceptual structures, we think that a geometric representation of these correlations is desirable.

Consider the left part of Figure \ref{fig:ConvexityProblem}. In this example, we consider two domains, age and height, in order to define the concepts of child and adult. We expect a strong correlation between age and height for children, but no such correlation for adults. This is represented by the two ellipses.\footnote{Please note that this is a very simplified artificial example to illustrate our main point.} As one can see, the values of age and height constrain each other: For instance, if the value on the age dimension is low and the point lies in the child region, then also the value on the height dimension must be low.

\begin{figure}[tp]
\centering
\includegraphics[width=\columnwidth]{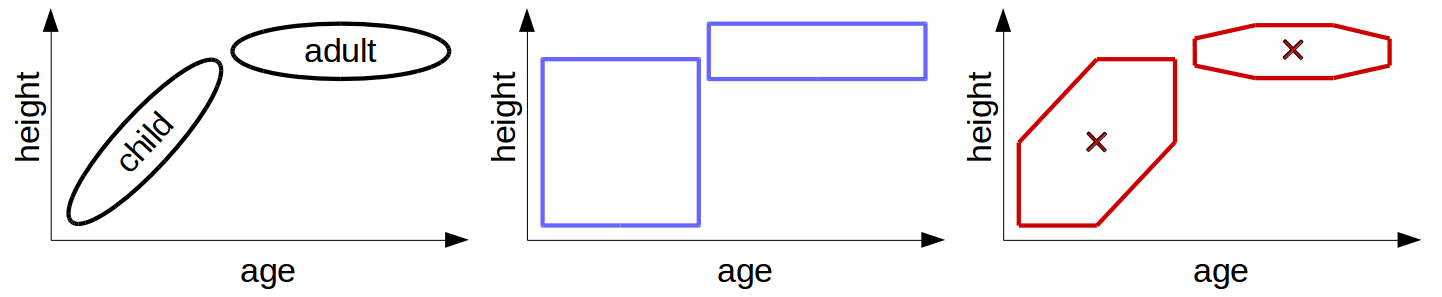}
\caption{Left: Intuitive way to define regions for the concepts of ``adult'' and ``child''. Middle: Representation by using convex sets. Right: Representation by using star-shaped sets with central points marked by crosses.}
\label{fig:ConvexityProblem}
\end{figure}

Domains are combined by using the Manhattan metric and convex sets under the Manhattan metric are axis-parallel cuboids. As the dimensions of age and height belong to different domains, a convex representation of the two concepts results in the rectangles shown in the middle of Figure \ref{fig:ConvexityProblem}. As one can see, all information about the correlation of age and height in the child concept is lost in this representation: The values of age and height do not constrain each other at all. According to this representation, a child of age 2 with a height of 1.60 m would be totally conceivable. This illustrates that we cannot geometrically represent correlations between domains if we assume that concepts are convex and that the Manhattan metric is used. We think that our example is not a pathological one and that similar problems will occur quite frequently when encoding concepts.\footnote{For instance, there is an obvious correlation between a banana's color and its taste. If you replace the ``age'' dimension with ``hue'' and the ``height'' dimension with ``sweetness'' in Figure \ref{fig:ConvexityProblem}, you will observe similar encoding problems for the ``banana'' concept as for the ``child'' concept.}\\

Star-shapedness is a weaker constraint than convexity. If we only require concepts to be star-shaped under the Manhattan metric, we can represent the correlation of age and height for children in a geometric way. This is shown in the right part of Figure \ref{fig:ConvexityProblem}: Both sketched sets are star-shaped under the Manhattan metric with respect to a central point.\footnote{Please note that although the sketched sets are still convex under the Euclidean metric, they are star-shaped but not convex under the Manhattan metric.} Although the star-shaped sets do not exactly correspond to our intuitive sketch in the left part of Figure \ref{fig:ConvexityProblem}, they definitely are an improvement over the convex representation.

By definition, star-shaped sets cannot contain any ``holes''. They furthermore have a well defined central region $P$ that can be interpreted as a prototype. Thus, the connection that \cite{Gardenfors2000} established between the prototype theory of concepts and the framework of conceptual spaces is preserved. Replacing convexity with star-shapedness is therefore only a minimal departure from the original framework.\\

The problem illustrated in Figure \ref{fig:ConvexityProblem} could also be resolved by replacing the Manhattan metric with a diferent distance function for combining domains. A natural choice would be to use the Euclidean metric everywhere. We think however that this would be a major modification of the original framework. For instance, if the Euclidean metric is used everywhere, the usage of domains to structure the conceptual space loses its main effect of influencing the overall distance metric. 
Moreover, there exists some psychological evidence \cite{Attneave1950,Shepard1964,Shepard1987,Johannesson2001} which indicates that human similarity ratings are reflected better by the Manhattan metric than by the Euclidean metric if different domains are involved (e.g., stimuli differing in size and brightness). As a psychologically plausible representation of similarity is one of the core principles of the conceptual spaces framework, these findings should be taken into account.
Furthermore, in high-dimensional feature spaces the Manhattan metric provides a better relative contrast between close and distant points than the Euclidean metric \cite{Aggarwal2001}. If we expect a large number of domains, this also supports the usage of the Manhattan metric from an implementational point of view.
One could of course also replace the Manhattan metric with something else (e.g., the Mahanalobis distance). However, as there is currently no strong evidence supporting the usage of any particular other metric, we think that it is best to use the metrics proposed in the original conceptual spaces framework.

Based on these arguments, we think that relaxing the convexity constraint is a better option than abolishing the use of the Manhattan metric.\\

Please note that the example given above is intended to highlight representational problems of the conceptual spaces framework if it is applied to artificial intelligence and if a geometric representation of correlations is desired. We do not make any claims that star-shapedness is a psychologically plausible extension of the original framework and we do not know about any psychological data that could support such a claim.

\section{A Parametric Definition of Concepts}
\label{FSSSS}

\subsection{Preliminaries}
\label{FSSSS:Preliminaries}

Our formalization is based on the following insight: 

\begin{lemma}
\label{lemma:UnionOfConvex}
Let $C_1, ..., C_m$ be convex sets in $CS$ under some metric $d$ and let $P := \textstyle\bigcap_{i=1}^{m} C_i$. If $P \neq \emptyset$, then $S := \textstyle\bigcup_{i=1}^{m} C_i$ is star-shaped under $d$ with respect to $P$.
\end{lemma}
\begin{proof}
Obvious (see also \cite{Smith1968}).
\end{proof}

We will use axis-parallel cuboids as building blocks for our star-shaped sets. They are defined in the following way:\footnote{All of the following definitions and propositions hold for any number of dimensions.}

\begin{definition}
\label{def:Cuboid}
(Axis-parallel cuboid)\\
We describe an \emph{axis-parallel cuboid}\footnote{We will drop the modifier "axis-parallel" from now on.} $C$ as a triple $\left\langle\Delta_C, p^-, p^+\right\rangle$. $C$ is defined on the domains $\Delta_C \subseteq \Delta$, i.e., on the dimensions $D_C := \textstyle\bigcup_{\delta \in \Delta_C} \delta$. We call $p^-, p^+ $ the support points of $C$ and require that:
$$
\forall {d \in D_C}: p^+_d,p^-_d \notin \left\{+\infty,-\infty\right\} \quad \land \quad
\forall {d \in D \setminus D_C}: p^-_d := -\infty \land p^+_d := +\infty
$$
Then, we define the cuboid $C$ in the following way:
$$C = \left\{x \in CS \;|\; \forall {d \in D}: p^-_d \leq x_d \leq p^+_d\right\}$$
\end{definition}

\begin{lemma}
\label{lemma:Cuboid}
A cuboid $C$ is convex under $d_C^{\Delta}$, given a fixed set of weights $W$.
\end{lemma}
\begin{proof}
It is easy to see that cuboids are convex with respect to $d_M$ and $d_E$. Based on this, one can show that they are also convex with respect to $d_C^{\Delta}$, which is a combination of $d_M$ and $d_E$.
\end{proof}

Our formalization will make use of fuzzy sets \cite{Zadeh1965}, which can be defined in our current context as follows:

\begin{definition}
(Fuzzy set)\\
A \emph{fuzzy set} $\widetilde{A}$ on $CS$ is defined by its membership function $\mu_{\widetilde{A}}: CS \rightarrow [0,1]$.
\end{definition}

For each $x \in CS$, we interpret $\mu_{\widetilde{A}}\left(x\right)$ as degree of membership of $x$ in $\widetilde{A}$. Note that fuzzy sets contain crisp sets as a special case where $\mu_{\widetilde{A}}: CS \rightarrow \left\{0,1\right\}$.

\begin{definition}
(Alpha-cut)\\
Given a fuzzy set $\widetilde{A}$ on $CS$, its \emph{$\alpha$-cut} ${\widetilde{A}}^{\alpha}$ for $\alpha \in [0,1]$ is defined as follows:
$${\widetilde{A}}^{\alpha} := \left\{x \in CS\; |\; \mu_{\widetilde{A}}\left(x\right) \geq \alpha\right\}$$
\end{definition}

\begin{definition}
\label{def:FuzzyStarShaped}
(Fuzzy star-shapedness)\\
A fuzzy set $\widetilde{A}$ is called \emph{star-shaped} under a metric $d$ with respect to a crisp set $P$ if and only if all of its $\alpha$-cuts ${\widetilde{A}}^{\alpha}$ are either empty or star-shaped under $d$ with respect to $P$.
\end{definition}

One can also generalize the ideas of subsethood, intersection, and union from crisp to fuzzy sets. We adopt the most widely used definitions:

\begin{definition}
\label{def:FuzzyOperations}
(Operations on fuzzy sets)\\
Let $\widetilde{A}, \widetilde{B}$ be two fuzzy sets defined on $CS$.
\begin{itemize}
	\item Subsethood: \hspace{0.20cm}$\widetilde{A} \subseteq \widetilde{B} :\iff \left(\forall {x \in CS}: \mu_{\widetilde{A}}\left(x\right) \leq \mu_{\widetilde{B}}\left(x\right)\right)$
	\item Intersection: \hspace{0.15cm}$\forall x \in CS: \mu_{\widetilde{A} \cap \widetilde{B}}\left(x\right) := \min\left(\mu_{\widetilde{A}}\left(x\right),\mu_{\widetilde{B}}\left(x\right)\right)$
	\item Union: \hspace{1.2cm}$\forall x \in CS: \mu_{\widetilde{A} \cup \widetilde{B}}\left(x\right) := \max\left(\mu_{\widetilde{A}}\left(x\right),\mu_{\widetilde{B}}\left(x\right)\right)$
\end{itemize}
\end{definition}

\subsection{Fuzzy Simple Star-Shaped Sets}
\label{FSSSS:FSSSS}

By combining Lemma \ref{lemma:UnionOfConvex} and Lemma \ref{lemma:Cuboid}, we see that any union of intersecting cuboids is star-shaped under $d_C^{\Delta}$. We use this insight to define simple star-shaped sets (illustrated in Figure \ref{fig:FSSSS} with $m=3$ cuboids), which will serve as cores for our concepts:

\begin{definition}
\label{def:SSSS}
(Simple star-shaped set)\\
We describe a \emph{simple star-shaped set} $S$ as a tuple $\left\langle\Delta_S,\left\{C_1,\dots,C_m\right\}\right\rangle$ where $\Delta_S \subseteq \Delta$ is a set of domains on which the cuboids $\left\{C_1,\dots,C_m\right\}$ (and thus also $S$) are defined. We further require that the central region $P :=\textstyle\bigcap_{i = 1}^m C_i \neq \emptyset$. Then the simple star-shaped set $S$ is defined as 
$$S := \bigcup_{i=1}^m C_i$$
\end{definition}

\begin{figure}[tp]
\centering
\includegraphics[width = \columnwidth]{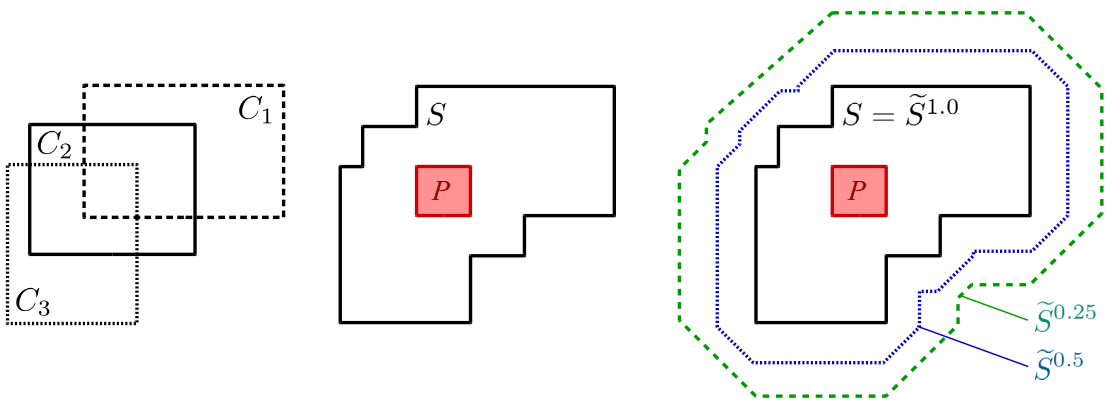} 
\caption{Left: Three cuboids $C_1, C_2, C_3$ with nonempty intersection. Middle: Resulting simple star-shaped set $S$ based on these cuboids. Right: Fuzzy simple star-shaped set $\tilde{S}$ based on $S$ with three $\alpha$-cuts for $\alpha \in \left\{1.0,0.5,0.25\right\}$.}
\label{fig:FSSSS}
\end{figure}

In practice, it is often not possible to define clear-cut boundaries for concepts and properties. It is, for example, very hard to define a generally accepted crisp boundary for the property ``red''. We therefore use a fuzzified version of simple star-shaped sets for representing concepts, which allows us to define imprecise concept boundaries. This usage of fuzzy sets for representing concepts has already a long tradition \cite{Osherson1982,Zadeh1982,Ruspini1991,Bvelohlavek2011,Douven2011}.
We use a simple star-shaped set $S$ as a concept's ``core'' and define the membership of any point $x \in CS$ to this concept as $\max_{y \in S}Sim\left(x,y\right)$:
\begin{definition}
\label{def:FSSSS}
(Fuzzy simple star-shaped set)\\
A \emph{fuzzy simple star-shaped set} $\widetilde{S}$ is described by a quadruple $\left\langle S,\mu_0,c,W\right\rangle$ where
$S = \left\langle\Delta_S,\left\{C_1,\dots,C_m\right\}\right\rangle$ is a non-empty simple star-shaped set. The parameter $\mu_0 \in (0,1]$ determines the highest possible membership to $\widetilde{S}$ and is usually set to 1. The sensitivity parameter $c > 0$ controls the rate of the exponential decay in the similarity function. Finally, $W = \left\langle W_{\Delta_S},\left\{W_{\delta}\right\}_{\delta \in \Delta_S}\right\rangle$ contains positive weights for all domains in $\Delta_S$ and all dimensions within these domains, reflecting their respective importance. We require that $\textstyle\sum_{\delta \in \Delta_S} w_{\delta} = |\Delta_S|$ and that $\forall {\delta \in \Delta_S}:\textstyle\sum_{d \in \delta} w_{d} = 1$.
The membership function of $\widetilde{S}$ is then defined as follows:
$$\mu_{\widetilde{S}}\left(x\right) := \mu_0 \cdot \max_{y \in S}\left(e^{-c \cdot d_C^{\Delta_S}\left(x,y,W\right)}\right)$$
\end{definition}

The sensitivity parameter $c$ controls the overall degree of fuzziness of $\widetilde{S}$ by determining how fast the membership drops to zero (larger values of $c$ result in steeper drops of the membership function). The weights $W$ represent not only the relative importance of the respective domain or dimension for the represented concept, but they also influence the relative fuzziness with respect to this domain or dimension (again, larger weights cause steeper drops).
Note that if $|\Delta_S| = 1$, then $\widetilde{S}$ represents a property, and if $|\Delta_S| > 1$, then $\widetilde{S}$ represents a concept.

The right part of Figure \ref{fig:FSSSS} shows a fuzzy simple star-shaped set $\widetilde{S}$. In this illustration, the $x$ and $y$ axes are assumed to belong to different domains, and are combined with the Manhattan metric using equal weights.

\begin{lemma}
\label{lemma:epsilonNeighborhood}
Let $\widetilde{S} = \left\langle S,\mu_0,c,W\right\rangle$ be a fuzzy simple star-shaped set and let $\alpha \leq \mu_0$. Then, $\widetilde{S}^\alpha$ is equivalent to an $\epsilon$-neighborhood of $S$ with $\epsilon = - \frac{1}{c} \cdot \ln\left(\frac{\alpha}{\mu_0}\right)$.
\end{lemma}
\begin{proof}\belowdisplayskip=-12pt
\begin{align*}
x \in \widetilde{S}^\alpha &\iff \mu_{\widetilde{S}}\left(x\right) = \mu_0 \cdot \max_{y \in S}\left(e^{-c \cdot d_C^{\Delta_S}\left(x,y,W\right)}\right) \geq \alpha \iff e^{-c \cdot \min_{y \in S} \left(d_C^{\Delta_S}\left(x,y,W\right)\right)} \geq \frac{\alpha}{\mu_0}\\
	&\iff \min_{y \in S} d_C^{\Delta_S}\left(x,y,W\right) \leq - \frac{1}{c} \cdot \ln\left(\frac{\alpha}{\mu_0}\right) =: \epsilon
\end{align*}
\end{proof}

\begin{proposition}
Any fuzzy simple star-shaped set $\widetilde{S} = \left\langle S,\mu_0,c,W\right\rangle$ is star-shaped with respect to $P = \textstyle\bigcap_{i=1}^{m} C_i$ under $d_C^{\Delta_S}$.
\end{proposition}
\begin{proof}
For $\alpha \leq \mu_0$, $\widetilde{S}^\alpha$ is an $\epsilon$-neighborhood of $S$ (Lemma \ref{lemma:epsilonNeighborhood}). We can define the $\epsilon$-neighborhood of a cuboid $C_i$ under $d_C^{\Delta_S}$ as
$$C_i^{\epsilon} = \left\{z \in CS\; |\; \forall {d \in D_S}: p^{-}_{id} - u_d \leq z_d \leq p^{+}_{id} + u_d\right\}$$
where the vector $u$ represents the difference between $x \in C_i$ and $z \in C_i^{\epsilon}$. Thus, $u$ must fulfill the following constraints:
$$\sum_{\delta \in \Delta_S} w_{\delta} \cdot \sqrt{\sum_{d \in \delta} w_{d} \cdot \left(u_{d}\right)^2} \leq \epsilon \quad \land \quad \forall d \in D_S: u_d \geq 0$$
Let now $x \in C_i,z \in C_i^{\epsilon}$. 
\begin{align*}
\forall {d \in D_S}:\;\; &x_d = p^{-}_{id} + a_d \;\text{with} \; a_d \in [0, p^{+}_{id} - p^{-}_{id}]\\
&z_d = p^{-}_{id} + b_d \;\;\text{with} \; b_d \in [-u_d, p^{+}_{id} - p^{-}_{id} + u_d]
\end{align*}
We know that a point $y \in CS$ is between $x$ and $z$ with respect to $d_C^{\Delta_S}$ if the following condition is true: 
\begin{align*}
&d_C^{\Delta_S}\left(x,y,W\right) + d_C^{\Delta_S}\left(y,z,W\right) = d_C^{\Delta_S}\left(x,z,W\right)\\
&\iff \forall {\delta \in \Delta_S}: d^{\delta}_E\left(x,y,W_{\delta}\right) + d^{\delta}_E\left(y,z,W_{\delta}\right) = d^{\delta}_E\left(x,z,W_{\delta}\right)\\
&\iff \forall {\delta \in \Delta_S}: \exists {t \in [0,1]}: \forall {{d} \in \delta}: y_{d} = t \cdot x_{d} + \left(1 - t\right) \cdot z_{d}
\end{align*}
The first equivalence holds because $d_C^{\Delta_S}$ is a weighted sum of Euclidean metrics $d^{\delta}_E$. As the weights are fixed and as the Euclidean metric is obeying the triangle inequality, the equation with respect to $d_C^{\Delta_S}$ can only hold if the equation with respect to $d_E^{\delta}$ holds for all $\delta \in \Delta$.

We can thus write the components of $y$ like this:
\begin{align*}
\forall d \in D_S: \exists t \in [0,1]: y_d &= t \cdot x_d + \left(1-t\right)\cdot z_d = t \cdot \left(p^{-}_{id} + a_d\right) + \left(1-t\right) \left(p^{-}_{id} + b_d\right)\\
&= p^{-}_{id} + t \cdot a_d + \left(1-t\right) \cdot b_d = p^{-}_{id} + c_d
\end{align*}

As $c_d := t \cdot a_d + \left(1-t\right) \cdot b_d \in [-u_d, p^{+}_{id} - p^{-}_{id} + u_d]$, it follows that $y \in C_i^{\epsilon}$. So $C_i^{\epsilon}$ is star-shaped with respect to $C_i$ under $d_C^{\Delta_S}$. More specifically, $C_i^{\epsilon}$ is also star-shaped with respect to $P$ under $d_C^{\Delta_S}$. Therefore, $S^{\epsilon} = \bigcup_{i=1}^{m} C_i^{\epsilon}$ is star-shaped under $d_C^{\Delta_S}$ with respect to $P$. Thus, all $\widetilde{S}^\alpha$ with $\alpha \leq \mu_0$ are star-shaped under $d_C^{\Delta_S}$ with respect to $P$. It is obvious that $\widetilde{S}^\alpha = \emptyset$ if $\alpha > \mu_0$, so $\widetilde{S}$ is star-shaped according to Definition \ref{def:FuzzyStarShaped}.
\end{proof}

From now on, we will use the terms ``core'' and ``concept'' to refer to our definitions of simple star-shaped sets and fuzzy simple star-shaped sets, respectively.
\section{Operations on Concepts}
\label{Operations}

In this section, we develop a number of operations, which can be used to create new concepts from existing ones and to describe relations between concepts.

\subsection{Intersection}
\label{Operations:Intersection}

The intersection of two concepts can be interpreted as logical conjunction: Intersecting ``green'' with ``banana'' should result in the concept ``green banana''.\\

If we intersect two \emph{cores} $S_1$ and $S_2$, we simply need to intersect their cuboids. As an intersection of two cuboids is again a cuboid, the result of intersecting two cores can be described as a union of cuboids. It is simple star-shaped if and only if these resulting cuboids have a nonempty intersection. This is only the case if the central regions $P_1$ and $P_2$ of $S_1$ and $S_2$ intersect.\footnote{Note that if the two cores are defined on completely different domains (i.e., $\Delta_{S_1} \cap \Delta_{S_2} = \emptyset$), then their central regions intersect (i.e., $P_1 \cap P_2 \neq \emptyset$), because we can find at least one point in the overall space that belongs to both $P_1$ and $P_2$.}

However, we would like our intersection to result in a valid core even if $P_1 \cap P_2 = \emptyset$. Thus, when intersecting two cores, we might need to apply some repair mechanism in order to restore star-shapedness.\\

We propose to extend the cuboids $C_i$ of the intersection in such a way that they meet in some ``midpoint'' $p^* \in CS$ (e.g., the arithmetic mean of their centers). We create extended versions $C_i^{*}$ of all $C_i$ by defining their support points like this: 
$$\forall {d \in D}: p_{id}^{-*} := \min\left(p_{id}^-, p^*_d\right), \quad p_{id}^{+*} := \max\left(p_{id}^+, p^*_d\right)$$

The intersection of the resulting $C^{*}_i$ contains at least $p^*$, so it is not empty. This means that $S' = \left\langle\Delta_{S_1} \cup \Delta_{S_2}, \left\{C_1^{*},\dots,C_{m}^{*}\right\}\right\rangle$ is again a simple star-shaped set. We denote this modified intersection (consisting of the actual intersection and the application of the repair mechanism) as $S' = I\left(S_1,S_2\right).$\\

\noindent
We define the intersection of two \emph{concepts} as $\widetilde{S}' = I(\widetilde{S}_1,\widetilde{S}_2) := \left\langle S',\mu'_0,c',W'\right\rangle$ with:
\begin{itemize}
	\item $S' := I\left(\widetilde{S}_1^{\alpha'},{\widetilde{S}}_2^{\alpha'}\right)$ (where $\alpha' = \max\left\{\alpha \in [0,1]: {\widetilde{S}}_1^{\alpha} \cap {\widetilde{S}}_2^{\alpha} \neq \emptyset\right\}$)
	\item $\mu'_0 := \alpha'$
	\item $c' := \min\left(c^{\left(1\right)},c^{\left(2\right)}\right)$
	\item $W'$ with weights defined as follows (where $s,t \in [0,1]$)\footnote{In some cases, the normalization constraint of the resulting domain weights might be violated. We can enforce this constraint by manually normalizing the weights afterwards.}:
	\begin{align*}
&\forall {\delta \in \Delta_{S_1} \cap \Delta_{S_2}}: \left(\left(w'_{\delta} := s \cdot w^{\left(1\right)}_{\delta} + \left(1-s\right) \cdot w^{\left(2\right)}_{\delta}\right)\right.\\
& \hspace{3.0cm}\left.\land \forall {d \in \delta}: \left(w'_{d} := t \cdot w^{\left(1\right)}_{d} + \left(1-t\right) \cdot w^{\left(2\right)}_{d}\right)\right)\\
&\forall {\delta \in \Delta_{S_1} \setminus \Delta_{S_2}}: \left(\left(w'_{\delta} := w^{\left(1\right)}_{\delta}\right) \land \forall {d \in \delta}: \left(w'_{d} := w^{\left(1\right)}_{d}\right)\right)\\[-2pt]
&\forall {\delta \in \Delta_{S_2} \setminus \Delta_{S_1}}: \left(\left(w'_{\delta} := w^{\left(2\right)}_{\delta}\right) \land \forall {d \in \delta}: \left(w'_{d} := w^{\left(2\right)}_{d}\right)\right)
	\end{align*}
\end{itemize}

When taking the combination of two somewhat imprecise concepts, the result should not be more precise than any of the original concepts. As the sensitivity parameter $c$ is inversely related to fuzziness, we take the minimum. If a weight is defined for both original concepts, we take a convex combination, and if it is only defined for one of them, we simply copy it. The importance of each dimension and domain to the new concept will thus lie somewhere between its importance with respect to the two original concepts.\\

\begin{figure}[tp]
\centering
\includegraphics[width=\textwidth]{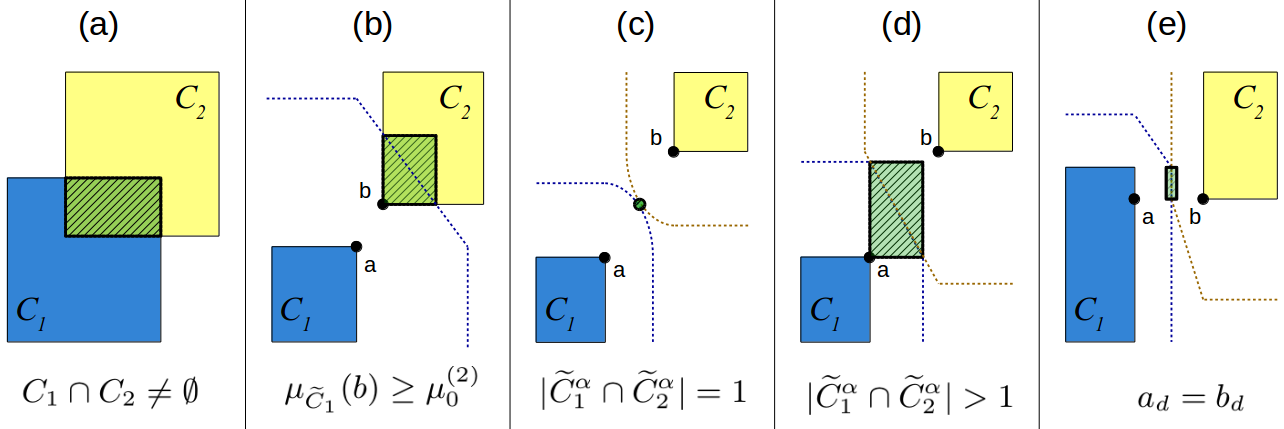}
\caption{Possible results of intersecting two fuzzy cuboids.}
\label{fig:IntersectionResults}
\end{figure}

The key challenge with respect to the intersection of two concepts $\widetilde{S}_1$ and $\widetilde{S}_2$ is to find the new core $S'$, i.e., the highest non-empty $\alpha$-cut intersection of the two concepts. We simplify this problem by iterating over all combinations of cuboids $C_1 \in S_1, C_2 \in S_2$ and by looking at each pair of cuboids individually. Let $a \in C_1$ and $b \in C_2$ be the two closest points from the two cuboids under consideration (i.e., $\forall{x \in C_1, y \in C_2}: d\left(a,b\right) \leq d\left(x,y\right)$). Let us define for a cuboid $C \in S$ its fuzzified version $\widetilde{C}$ as follows (cf. Definition \ref{def:FSSSS}):
$$\mu_{\widetilde{C}}\left(x\right) = \mu_0 \cdot \max_{y \in C}\left(e^{-c \cdot d_C^{\Delta}\left(x,y,W\right)}\right)$$
It is obvious that $\mu_{\widetilde{S}}\left(x\right) = \max_{C_i \in S} \mu_{\widetilde{C}_i}\left(x\right)$. When intersecting two fuzzified cuboids $\widetilde{C}_1$ and $\widetilde{C}_2$, the following results are possible:
\begin{enumerate}
	\item The crisp cuboids have a nonempty intersection (Figure \ref{fig:IntersectionResults}a). In this case, we simply compute their crisp intersection. The $\alpha$-value of this intersection is equal to $\min(\mu_0^{\left(1\right)},\mu_0^{\left(2\right)})$.
	\item The $\mu_0$ parameters are different and the $\mu_0^{\left(i\right)}$-cut of $\widetilde{C}_j$ intersects with $C_i$ (Figure \ref{fig:IntersectionResults}b). In this case, we need to intersect $\widetilde{C}_j^{\mu_0^{\left(i\right)}}$ with $C_i$ and approximate the result by a cuboid. The $\alpha$-value of this intersection is equal to $\mu_0^{\left(i\right)}$.
	\item The intersection of the two fuzzified cuboids consists of a single point $x^*$ lying between $a$ and $b$ (Figure \ref{fig:IntersectionResults}c). In this case, we define a trivial cuboid with $p^- = p^+ = x^*$. The $\alpha$-value of this intersection is $\mu_{\widetilde{C}_1}\left(x^*\right) = \mu_{\widetilde{C}_2}\left(x^*\right)$.
	\item The intersection of the two fuzzified cuboids consists of a set of points (Figure \ref{fig:IntersectionResults}d). This can only happen if the $\alpha$-cut boundaries of both fuzzified cuboids are parallel to each other, which requires multiple domains to be involved and the weights of both concepts to be linearly dependent. Again, we approximate this intersection by a cuboid. We obtain the $\alpha$-value of this intersection by computing $\mu_{\widetilde{C}_1}\left(x\right) = \mu_{\widetilde{C}_2}\left(x\right)$ for some $x$ in the set of points obtained in the beginning.
\end{enumerate}

Moreover, it might happen that both $a$ and $b$ can vary in a certain dimension $d$ (Figure \ref{fig:IntersectionResults}e). In this case, the resulting cuboid needs to be extruded in this dimension.\\

After having computed a cuboid approximation of the intersection for all pairs of fuzzified cuboids, we aggregate them by removing all intersection results with non-maximal $\alpha$. If the remaining set of cuboids has an empty intersection, we perform the repair mechanism as defined above.

\subsection{Unification}
\label{Operations:Union}

A unification of two or more concepts can be used to construct higher-level concepts. For instance, the concept of ``citrus fruit'' can be obtained by unifying ``lemon'', ``orange'', ``grapefruit'', ``lime'', etc.\\

As each core is defined as a union of cuboids, the unification of two \emph{cores} can also be expressed as a union of cuboids. The resulting set is star-shaped if and only if the central regions of the original cores intersect. So after each unification, we might again need to perform a repair mechanism in order to restore star-shapedness. We propose to use the same repair mechanism that is also used for intersections. We denote the modified unification as $S' = U\left(S_1, S_2\right)$.\\

\noindent
We define the unification of two \emph{concepts} as $\widetilde{S}' = U(\widetilde{S}_1, \widetilde{S}_2) := \left\langle S',\mu'_0,c',W'\right\rangle$ with:
\begin{itemize}
	\item $S' := U\left(S_1,S_2\right)$
	\item $\mu'_0 := \max\left(\mu_0^{\left(1\right)},\mu_0^{\left(2\right)}\right)$
	\item $c'$ and $W'$ as described in Section \ref{Operations:Intersection}
\end{itemize}

\begin{proposition}
Let $\widetilde{S}_1 = \langle S_1, \mu_0^{\left(1\right)}, c^{\left(1\right)}, W^{\left(1\right)}\rangle$ and $\widetilde{S}_2 = \langle S_2, \mu_0^{\left(2\right)}, c^{\left(2\right)}, W^{\left(2\right)}\rangle$ be two concepts. If we assume that $\Delta_{S_1} = \Delta_{S_2}$ and $W^{\left(1\right)} = W^{\left(2\right)}$, then $\widetilde{S}_1 \cup \widetilde{S}_2 \subseteq U(\widetilde{S}_1, \widetilde{S}_2) = \widetilde{S}'$.
\end{proposition}
\begin{proof}
As both $\Delta_{S_1} = \Delta_{S_2}$ and $W^{\left(1\right)} = W^{\left(2\right)}$, we know that 
$$d\left(x,y\right) := d_C^{\Delta_{S_1}}\left(x,y,W^{\left(1\right)}\right) = d_C^{\Delta_{S_2}}\left(x,y,W^{\left(2\right)}\right) = d_C^{\Delta_{S'}}\left(x,y,W'\right)$$
Moreover, $S_1 \cup S_2 \subseteq S'$. Therefore:
\belowdisplayskip=-12pt \begin{align*}
\mu_{\widetilde{S}_1 \cup \widetilde{S}_2}\left(x\right) &= \max\left(\mu_{\widetilde{S}_1}\left(x\right),\mu_{\widetilde{S}_2}\left(x\right)\right)\\
&= \max\left(\max_{y \in S_1} \left(\mu_0^{\left(1\right)} \cdot e^{-c^{\left(1\right)} \cdot d\left(x,y\right)}\right),\max_{y \in S_2} \left(\mu_0^{\left(2\right)} \cdot e^{-c^{\left(2\right)} \cdot d\left(x,y\right)}\right)\right)\\
&\leq \mu'_0 \cdot \max\left(e^{-c^{\left(1\right)} \cdot \min_{y \in S_1} d\left(x,y\right)}, e^{-c^{\left(2\right)} \cdot \min_{y \in S_2} d\left(x,y\right)}\right)\\
&\leq \mu'_0 \cdot e^{-c' \cdot \min\left(\min_{y \in S_1} d\left(x,y\right),\; \min_{y \in S_2} d\left(x,y\right)\right)}\\
&\leq \mu'_0 \cdot e^{-c' \cdot \min_{y \in S'} d\left(x,y\right)} = \mu_{\widetilde{S}'}\left(x\right)
\end{align*}
\end{proof}

\subsection{Subspace Projection}
\label{Operations:Projection}

Projecting a concept onto a subspace corresponds to focusing on certain domains while completely ignoring others. For instance, projecting the concept ``apple'' onto the color domain results in a property that describes the typical color of apples.\\

Projecting a cuboid onto a subspace results in a cuboid. As one can easily see, projecting a \emph{core} onto a subspace results in a valid core. We denote the projection of a core $S$ onto domains $\Delta_{S'} \subseteq \Delta_S$ as $S' = P(S, \Delta_{S'})$.\\

We define the projection of a \emph{concept} $\widetilde{S}$ onto domains $\Delta_{S'} \subseteq \Delta_S$ as $\widetilde{S}' = P(\widetilde{S}, \Delta_{S'}) := \left\langle S', \mu'_0, c', W'\right\rangle$ with:
\begin{itemize}
	\item $S' := P\left(S,\Delta_{S'}\right)$
	\item $\mu'_0 := \mu_0$
	\item $c' := c$
	\item $W' := \left\langle\left\{|\Delta_S'| \cdot \frac{w_{\delta}}{\sum_{\delta' \in \Delta_{S'}} w_{\delta'}}\right\}_{\delta \in \Delta_{S'}},\left\{W_{\delta}\right\}_{\delta \in \Delta_{S'}}\right\rangle$
\end{itemize}

Note that we only apply minimal changes to the parameters: $\mu_0$ and $c$ stay the same, only the domain weights are updated in order to not violate their normalization constraint.\\

Projecting a set onto two complementary subspaces and then intersecting these projections again in the original space yields a superset of the original set. This is intuitively clear for cores and can also be shown for concepts under one additional constraint:
\begin{proposition} 
Let $\widetilde{S} = \left\langle S, \mu_0, c, W\right\rangle$ be a concept. Let $\widetilde{S}_1 = P(\widetilde{S}, \Delta_1)$ and $\widetilde{S}_2 = P(\widetilde{S}, \Delta_2)$ with $\Delta_1 \cup \Delta_2 = \Delta_S$ and $\Delta_1 \cap \Delta_2 = \emptyset$. Let $\widetilde{S}' = I(\widetilde{S}_1, \widetilde{S}_2)$ as described in Section \ref{Operations:Intersection}. If $\sum_{\delta \in \Delta_1} w_{\delta} = |\Delta_1|$ and $\sum_{\delta \in \Delta_2} w_{\delta} = |\Delta_2|$, then $\widetilde{S} \subseteq \widetilde{S}'$.
\end{proposition}
\begin{proof}
We already know that $S \subseteq I\left(P\left(S, \Delta_1\right), P\left(S, \Delta_2\right)\right) = S'$. Moreover, one can easily see that $\mu'_0 = \mu_0$ and $c' = c$.
\begin{align*}
&\mu_{\widetilde{S}}\left(x\right) = \max_{y \in S} \left(\mu_0 \cdot e^{-c \cdot d_C^{\Delta_S}\left(x,y,W\right)}\right)\overset{!}{\leq} \max_{y \in S'} \left(\mu'_0 \cdot e^{-c' \cdot d_C^{\Delta_S}\left(x,y,W'\right)}\right) = \mu_{\widetilde{S}'}\left(x\right)
\end{align*}
This holds if and only if $W = W'$. $W^{\left(1\right)}$ only contains weights for $\Delta_1$, whereas $W^{\left(2\right)}$ only contains weights for $\Delta_2$. As $\sum_{\delta \in \Delta_i} w_{\delta} = |\Delta_i|$ (for $i \in \left\{1,2\right\}$), the weights are not changed during the projection. As $\Delta_1 \cap \Delta_2 = \emptyset$, they are also not changed during the intersection, so $W' = W$.
\end{proof}

\subsection{Axis-Parallel Cut}
\label{Operations:Cut}

In a concept formation process, it might happen that over-generalized concepts are learned (e.g., a single concept that represents both dogs and cats). If it becomes apparent that a finer-grained conceptualization is needed, the system needs to be able to split its current concepts into multiple parts.\\ 

One can split a concept $\widetilde{S} = \left\langle S, \mu_0, c, W \right\rangle$ into two parts by selecting a value $v$ on a dimension $d$ and by splitting each cuboid $C \in S$ into two child cuboids $C^{\left(+\right)} := \left\{x \in C \;|\; x_d \geq v\right\}$ and $C^{\left(-\right)} := \left\{x \in C \;|\; x_d \leq v\right\}$.\footnote{A strict inequality in the definition of $C^{\left(+\right)}$ or $C^{\left(-\right)}$ would not yield a cuboid.} Both $S^{\left(+\right)} := \bigcup_{i=1}^{m} C_i^{\left(+\right)}$ and $S^{\left(-\right)} := \bigcup_{i=1}^{m} C_i^{\left(-\right)}$ are still valid cores: They are both a union of cuboids and one can easily show that the intersection of these cuboids is not empty. We define $\widetilde{S}^{\left(+\right)} := \left\langle S^{\left(+\right)}, \mu_0, c, W \right\rangle$ and $\widetilde{S}^{\left(-\right)} := \left\langle S^{\left(-\right)}, \mu_0, c, W \right\rangle$, both of which are by definition valid concepts. Note that by construction, $S^{\left(-\right)} \cup S^{\left(+\right)} = S$ and $U(\widetilde{S}^{\left(-\right)}, \widetilde{S}^{\left(+\right)}) = \widetilde{S}$.

\subsection{Concept Size}
\label{Operations:Hypervolume}
The size of a concept gives an intuition about its specificity: Large concepts are more general and small concepts are more specific. This is one obvious aspect in which one can compare two concepts to each other.\\

One can use a measure $M$ to describe the size of a fuzzy set. It can be defined in our context as follows (cf. \cite{Bouchon-Meunier1996}):
\begin{definition}
A measure $M$ on a conceptual space $CS$ is a function $M: \mathcal{F}\left(CS\right) \rightarrow \mathbb{R}^+_0$ with $M\left(\emptyset\right) = 0$ and $\widetilde{A} \subseteq \widetilde{B} \Rightarrow M(\widetilde{A}) \leq M(\widetilde{B})$, where $\mathcal{F}\left(CS\right)$ is the fuzzy power set of $CS$.
\end{definition}

A common measure for fuzzy sets is the integral over the set's membership function, which is equivalent to the Lebesgue integral over the fuzzy set's $\alpha$-cuts:
\begin{equation}
M\left(\widetilde{A}\right) := \int_{CS} \mu_{\widetilde{A}}\left(x\right)\; dx = \int_{0}^1 V\left(\widetilde{A}^{\alpha}\right)\; d\alpha
\label{eqn:integral}
\end{equation}
We use $V(\widetilde{A}^{\alpha})$ to denote the volume of a fuzzy set's $\alpha$-cut.
One can easily see that we can use the inclusion-exclusion formula (cf. e.g., \cite{Bogart1989}) to compute the overall measure of a concept $\widetilde{S}$ based on the measure of its fuzzified cuboids\footnote{Note that the intersection of two overlapping fuzzified cuboids is again a fuzzified cuboid.}:
\begin{equation}
M\left(\widetilde{S}\right) = \sum_{l=1}^m \left(\left(-1\right)^{l+1} \cdot \sum_{\substack{\left\{i_1,\dots,i_l\right\}\\\subseteq\left\{1,\dots,m\right\}}}M\left(\bigcap_{i \in \left\{i_1,\dots,i_l\right\}} \widetilde{C}_i\right)\right)
\label{eqn:inclusionExclusion}
\end{equation}
The outer sum iterates over the number of cuboids under consideration (with $m$ being the total number of cuboids in S) and the inner sum iterates over all sets of exactly $l$ cuboids. The overall formula generalizes the observation that $|A\; \cup\; B| = |A| + |B| - |A\; \cap\; B|$ from two to $m$ sets.\\

In order to derive $M(\widetilde{C})$, we first describe how to compute $V(\widetilde{C}^{\alpha})$, i.e., the size of a fuzzified cuboid's $\alpha$-cut. Using Equation \ref{eqn:integral}, we can then derive $M(\widetilde{C})$, which we can in turn insert into Equation \ref{eqn:inclusionExclusion} to compute the overall size of $\widetilde{S}$.\\

Figure \ref{fig:2DAlphaCut} illustrates the $\alpha$-cut of a fuzzified two-dimensional cuboid both under $d_E$ (left) and under $d_M$ (right). From Lemma \ref{lemma:epsilonNeighborhood} we know that one can interpret each $\alpha$-cut as an $\epsilon$-neighborhood of the original $C$ with $\epsilon = -\frac{1}{c} \cdot \ln\left(\frac{\alpha}{\mu_0}\right)$.

\begin{figure}[tp]
\centering
\includegraphics[width = 0.8\textwidth]{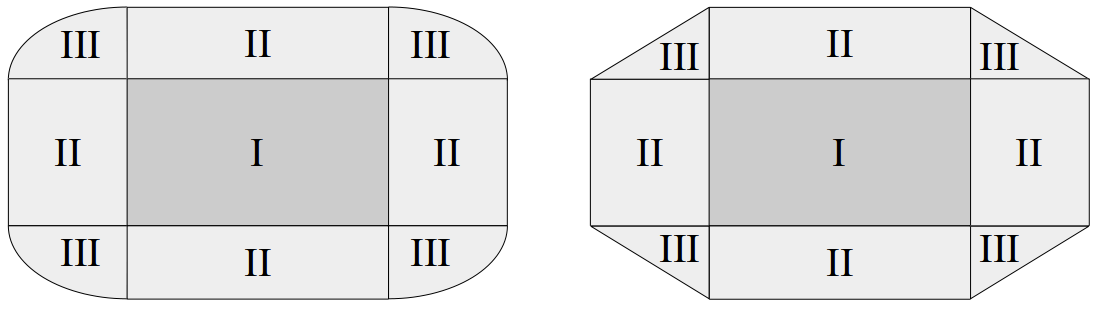}
\caption{$\alpha$-cut of a fuzzified cuboid under $d_E$ (left) and $d_M	$ (right), respectively.}
\label{fig:2DAlphaCut}
\end{figure}

$V(\widetilde{C}^{\alpha})$ can be described as a sum of different components. Let us use the shorthand notation $b_d := p_d^+ - p_d^-$.
Looking at Figure \ref{fig:2DAlphaCut}, one can see that all components of $V(\widetilde{C}^{\alpha})$ can be described by ellipses\footnote{Note that ellipses under $d_M$ have the form of streched diamonds.}: Component I is a zero-dimensional ellipse (i.e., a point) that was extruded in two dimensions with extrusion lengths of $b_1$ and $b_2$, respectively. Component II consists of two one-dimensional ellipses (i.e., line segments) that were extruded in one dimension, and component III is a two-dimensional ellipse. 

Let us denote by $\Delta_{\left\{d_1,\dots,d_i\right\}}$ the domain structure obtained by eliminating from $\Delta$ all dimensions $d \in D\setminus\left\{d_1, \dots, d_i\right\}$. Moreover, let $V\left(r, \Delta, W\right)$ be the hypervolume of a hyperball under $d_C^\Delta\left(\cdot,\cdot, W\right)$ with radius $r$. In this case, a hyperball is the set of all points with a distance of at most $r$ (measured by $d_C^\Delta\left(\cdot,\cdot, W\right)$) to a central point. Note that the weights $W$ can cause this ball to have the form of an ellipse. For instance, in Figure \ref{fig:2DAlphaCut}, we assume that $w_{d_1} < w_{d_2}$ which means that we allow larger differences with respect to $d_1$ than with respect to $d_2$. This causes the hyperballs to be streched in the $d_1$ dimension, thus obtaining the shape of an ellipse.
We can in general describe $V(\widetilde{C}^{\alpha})$ as follows:
$$V\left(\widetilde{C}^{\alpha}\right) = \sum_{i=0}^n \left( \sum_{\substack{\left\{d_1,\dots,d_i\right\}\\ \subseteq D}} \left( \prod_{\substack{d \in\\ D\setminus\left\{d_1,\dots,d_i\right\}}} b_d \right) \cdot V\left( -\frac{1}{c} \cdot \ln\left(\frac{\alpha}{\mu_0}\right),\Delta_{\left\{d_1, \dots, d_i\right\}},W \right)\right)$$
The outer sum of this formula runs over the number of dimensions with respect to which a given point $x \in \widetilde{C}^{\alpha}$ lies outside of $C$. We then sum over all combinations $\left\{d_1,\dots,d_i\right\}$ of dimensions for which this could be the case, compute the volume $V\left(\cdot,\cdot,\cdot\right)$ of the $i$-dimensional hyperball in these dimensions, and extrude this intermediate result in all remaining dimensions by multiplying with $\prod_{d \in D\setminus\left\{d_1,\dots,d_i\right\}} b_d$.

Let us illustrate this formula for the $\alpha$-cuts shown in Figure \ref{fig:2DAlphaCut}: For $i = 0$, we can only select the empty set for the inner sum, so we end up with $b_1 \cdot b_2$, which is the size of the original cuboid (i.e., component I). For $i = 1$, we can either pick $\left\{d_1\right\}$ or $\left\{d_2\right\}$ in the inner sum. For $\left\{d_1\right\}$, we compute the size of the left and right part of component II by multiplying $V( -\frac{1}{c} \cdot \ln(\frac{\alpha}{\mu_0}),\Delta_{\left\{d_1\right\}},W)$ (i.e., their combined width) with $b_2$ (i.e., their height). For $\left\{d_2\right\}$, we analogously compute the size of the upper and the lower part of component II. Finally, for $i = 2$, we can only pick $\left\{d_1, d_2\right\}$ in the inner sum, leaving us with $V( -\frac{1}{c} \cdot \ln(\frac{\alpha}{\mu_0}),\Delta ,W)$, which is the size of component III.
One can easily see that the formula for $V(\widetilde{C}^{\alpha})$ also generalizes to higher dimensions.\\

As we have shown in \cite{Bechberger2017Hyperball}, $V\left(r, \Delta, W\right)$ can be computed as follows, where $n_\delta = |\delta|$:
$$V\left(r,\Delta, W\right) = \frac{1}{\prod_{\delta \in \Delta} w_{\delta} \cdot \prod_{d \in \delta} \sqrt{w_d}} \cdot \frac{r^n}{n!} \cdot \prod_{\delta \in \Delta} \left(n_\delta! \cdot \frac{\pi^{\frac{n_\delta}{2}}}{\Gamma\left(\frac{n_\delta}{2}+1\right)}\right)$$

Defining $\delta\left(d\right)$ as the unique $\delta \in \Delta$ with $d \in \delta$, and $a_d := w_{\delta\left(d\right)} \cdot \sqrt{w_{d}} \cdot b_d \cdot c$, we can use this observation to rewrite $V(\widetilde{C}^{\alpha})$:
\begin{align*}
V\left(\widetilde{C}^\alpha\right) &=  
\frac{1}{c^n\prod_{d \in D} w_{\delta\left(d\right)} \sqrt{w_d}}
\sum_{i=0}^{n} \left( 
\frac{\left(-1\right)^i \cdot \ln\left(\frac{\alpha}{\mu_0}\right)^i}{i!} \cdot 
\sum_{\substack{\left\{d_1,\dots,d_i\right\}\\ \subseteq D}} 
\left(\prod_{\substack{d \in \\D \setminus \left\{d_1,\dots,d_i\right\}}} a_d\right) \cdot \right.\\
&\hspace{4.5cm}\left.\prod_{\substack{\delta \in \\ \Delta_{\left\{d_1,\dots,d_i\right\}}}} \left(
n_\delta! \cdot \frac{\pi^{\frac{n_\delta}{2}}}{\Gamma\left(\frac{n_\delta}{2}+1\right)}\right)\right)\\
\end{align*}
We can now solve Equation \ref{eqn:integral} to compute $M(\widetilde{C})$ by using the following lemma:

\begin{lemma}
\label{lemma:logIntegral}
$\forall n \in \mathbb{N}: \int_0^{1} \ln\left(x\right)^n dx = \left(-1\right)^n \cdot n!$ 
\end{lemma}
\begin{proof}
Substitute $x = e^t$ and $s = -t$, then apply the definition of the $\Gamma$ function.
\end{proof}

\begin{proposition}
\label{proposition:Measure}
The measure of a fuzzified cuboid $\widetilde{C}$ can be computed as follows:
\begin{align*}
M\left(\widetilde{C}\right) &= \frac{\mu_0}{c^n\prod_{d \in D} w_{\delta\left(d\right)} \sqrt{w_d}}
\sum_{i=0}^{n} \left( 
\sum_{\substack{\left\{d_1,\dots,d_i\right\}\\ \subseteq D}} 
\left(\prod_{\substack{d \in \\ D \setminus \left\{d_1,\dots,d_i\right\}}} a_d\right) \cdot \right.\\
&\hspace{5cm}\left.\prod_{\substack{\delta \in\\ \Delta_{\left\{d_1,\dots,d_i\right\}}}} \left(
n_\delta! \cdot \frac{\pi^{\frac{n_\delta}{2}}}{\Gamma\left(\frac{n_\delta}{2}+1\right)}\right)\right)
\end{align*}
\end{proposition}
\begin{proof}
Substitute $x = \frac{\alpha}{\mu_0}$ in Equation \ref{eqn:integral} and apply Lemma \ref{lemma:logIntegral}.
\end{proof}

Although the formula for $M(\widetilde{C})$ is quite complex, it can be easily implemented via a set of nested loops. As mentioned earlier, we can use the result from Proposition \ref{proposition:Measure} in combination with the inclusion-exclusion formula (Equation \ref{eqn:inclusionExclusion}) to compute $M(\widetilde{S})$ for any concept $\widetilde{S}$. Also Equation \ref{eqn:inclusionExclusion} can be easily implemented via a set of nested loops.
Note that $M(\widetilde{S})$ is always computed only on $\Delta_S$, i.e., the set of domains on which $\widetilde{S}$ is defined.

\subsection{Subsethood}
\label{Operations:Subsethood}
In order to represent knowledge about a hierarchy of concepts, one needs to be able to determine whether one concept is a subset of another concept. For instance, the fact that $\widetilde{S}_{Granny Smith} \subseteq \widetilde{S}_{apple}$ indicates that Granny Smith is a hyponym of apple.\footnote{One could also say that the fuzzified cuboids $\widetilde{C}_i$ are sub-concepts of $\widetilde{S}$, because $\widetilde{C}_i \subseteq \widetilde{S}$.}\\

\noindent
The classic definition of subsethood for fuzzy sets reads as follows (cf. Definition \ref{def:FuzzyOperations}):
$$\widetilde{A} \subseteq \widetilde{B} :\iff \forall {x \in CS}: \mu_{\widetilde{A}}\left(x\right) \leq \mu_{\widetilde{B}}\left(x\right)$$
This definition has the weakness of only providing a binary/crisp notion of subsethood. It is desirable to define a \emph{degree} of subsethood in order to make more fine-grained distinctions. 
Many of the definitions for degrees of subsethood proposed in the fuzzy set literature \cite{Bouchon-Meunier1996,Young1996} require that the underlying universe is discrete. The following definition \cite{Kosko1992} works also in a continuous space and is conceptually quite straightforward: 
$$Sub\left(\widetilde{A},\widetilde{B}\right) = \frac{M\left(\widetilde{A} \cap \widetilde{B}\right)}{M\left(\widetilde{A}\right)} \quad \text{with a measure } M$$
One can interpret this definition intuitively as the ``percentage of $\widetilde{A}$ that is also in $\widetilde{B}$''. It can be easily implemented based on the intersection defined in Section \ref{Operations:Intersection} and the measure defined in Section \ref{Operations:Hypervolume}: 
$$Sub\left(\widetilde{S}_1,\widetilde{S}_2\right) := \frac{M\left(I\left(\widetilde{S}_1, \widetilde{S}_2\right)\right)}{M\left(\widetilde{S}_1\right)}$$ 
If $\widetilde{S}_1$ and $\widetilde{S}_2$ are not defined on the same domains, then we first project them onto their shared subset of domains before computing their degree of subsethood.\\

When computing the intersection of two concepts with different sensitivity parameters $c^{\left(1\right)}, c^{\left(2\right)}$ and different weights $W^{\left(1\right)}, W^{\left(2\right)}$, one needs to define new parameters $c'$ and $W'$ for the resulting concept. In Section \ref{Operations:Intersection}, we have argued that the sensitivity parameter $c'$ should be set to the minimum of $c^{\left(1\right)}$ and $c^{\left(2\right)}$. Now if $c^{\left(1\right)} > c^{\left(2\right)}$, then $c' = \min\left(c^{\left(1\right)}, c^{\left(2\right)}\right) = c^{\left(2\right)} < c^{\left(1\right)}$. It might thus happen that $M(I(\widetilde{S}_1, \widetilde{S}_2)) > M(\widetilde{S}_1)$, and that therefore $Sub(\widetilde{S}_1,\widetilde{S}_2) > 1$. As we would like to confine $Sub(\widetilde{S}_1,\widetilde{S}_2)$ to the interval $[0,1]$, we should use the same $c$ and $W$ for computing both $M(I(\widetilde{S}_1,\widetilde{S}_2))$ and $M(\widetilde{S}_1)$. 

When judging whether $\widetilde{S}_1$ is a subset of $\widetilde{S}_2$, we can think of $\widetilde{S}_2$ as setting the context by determining the relative importance of the different domains and dimensions as well as the degree of fuzziness. For instance, when judging whether tomatoes are vegetables, we focus our attention on the features that are crucial to the definition of the ``vegetable'' concept. We thus propose to use $c^{\left(2\right)}$ and $W^{\left(2\right)}$ when computing $M(I(\widetilde{S}_1,\widetilde{S}_2))$ and $M(\widetilde{S}_1)$.

\subsection{Implication}
\label{Operations:Implication}

Implications play a fundamental role in rule-based systems and all approaches that use formal logics for knowledge representation. It is therefore desirable to define an implication function on concepts, such that one is able to express facts like $apple \Rightarrow red$ within our formalization.\\

In the fuzzy set literature \cite{Mas2007}, a fuzzy implication is defined as a generalization of the classical crisp implication. Computing the implication of two fuzzy sets typically results in a new fuzzy set which describes the local validity of the implication for each point in the space. In our setting, we are however more interested in a single number that indicates the overall validity of the implication $apple \Rightarrow red$.
We propose to reuse the definition of subsethood from Section \ref{Operations:Subsethood}: It makes intuitive sense in our geometric setting to say that $apple \Rightarrow red$ is true to the degree to which $apple$ is a subset of $red$. We therefore define:
$$Impl\left(\widetilde{S}_1,\widetilde{S}_2\right) := Sub\left(\widetilde{S}_1,\widetilde{S}_2\right)$$

\subsection{Similarity and Betweenness}
\label{Operations:SimilarityBetweenness}

Similarity and betweenness of concepts can be valuable sources of information for common-sense reasoning \cite{Derrac2015}: If two concepts are similar, they are expected to have similar properties and behave in similar ways (e.g., pencils and crayons). If one concept (e.g., ``master student'') is conceptually between two other concepts (e.g., ``bachelor student'' and ``PhD student''), it is expected to share all properties and behaviors that the two other concepts have in common (e.g., having to pay an enrollment fee).\\

In Section \ref{CS:Definition}, we have already provided definitions for similarity and betweenness of points. We can naively define similarity and betweenness for concepts by applying the definitions from Section \ref{CS:Definition} to the midpoints of the concepts' central regions $P$ (cf. Definition \ref{def:SSSS}). For computing the similarity, we propose to use both the dimension weights and the sensitivity parameter of the second concept, which again in a sense provides the context for the similarity judgement. If the two concepts are defined on different sets of domains, we use only their common subset of domains for computing the distance of their midpoints and thus their similarity. Betweenness is a binary relation and independent of dimension weights and sensitivity parameters.

These proposed definitions are clearly very naive and shall be replaced by more sophisticated definitions in the future. Especially a graded notion of betweenness would be desirable.

\section{Implementation and Example}
\label{Implementation}

\begin{figure}[t]
\centering
\includegraphics[width=0.95\textwidth]{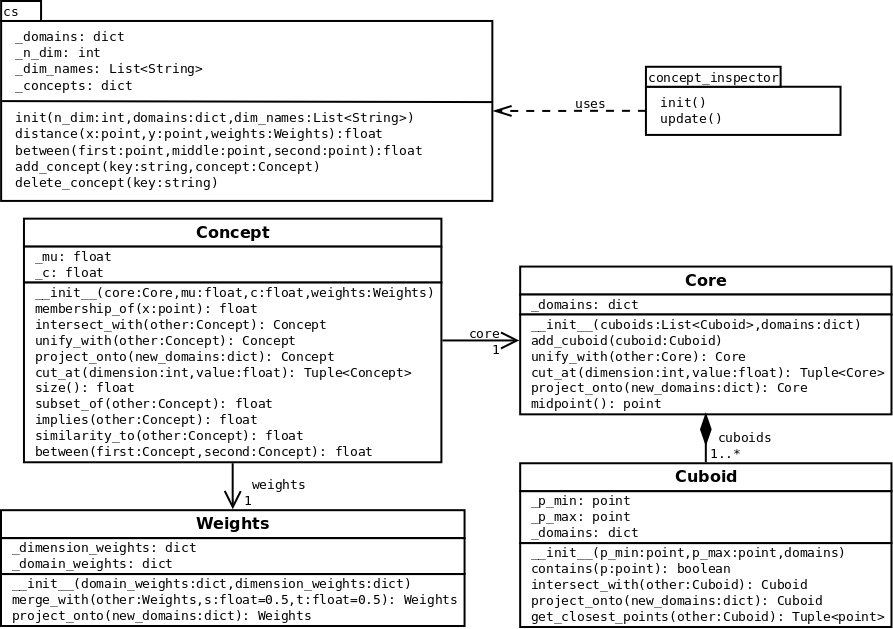}
\caption{Class diagram of our implementation.}
\label{fig:classDiagram}
\end{figure}
We have implemented our formalization in Python 2.7 and have made its source code publicly avaliable on GitHub\footnote{See \url{https://github.com/lbechberger/ConceptualSpaces}.} \cite{Bechberger2018GitHub}. Figure \ref{fig:classDiagram} shows a class diagram illustrating the overall structure of our implementation. As one can see, each of the components from our definition (i.e., weights, cuboids, cores, and concepts) is represented by an individual class. Moreover, the ``cs'' module contains the overall domain structure of the conceptual space (represented as a dictionary mapping from domain identifiers to sets of dimensions) along with some utility functions (e.g., computing distance and betweenness of points). The ``concept\_inspector'' package contains a visualization tool that displays 3D and 2D projections of the concepts stored in the ``cs'' package. When defining a new concept from scratch, one needs to use all of the classes, as all components of the concept need to be specified in detail. When operating with existing concepts, it is however sufficient to use the \texttt{Concept} class which contains all the operations defined in Section \ref{Operations}.\\

Our implementation of the conceptual spaces framework contains a simple toy example -- a three-dimensional conceptual space for fruits, defined as follows:
$$\Delta = \left\{\delta_{color} = \left\{d_{hue}\right\},\delta_{shape} = \left\{d_{round}\right\},\delta_{taste} = \left\{d_{sweet}\right\}\right\}$$
$d_{hue}$ describes the hue of the observation's color, ranging from $0.00$ (purple) to $1.00$ (red). $d_{round}$ measures the percentage to which the bounding circle of an object is filled. $d_{sweet}$ represents the relative amount of sugar contained in the fruit, ranging from 0.00 (no sugar) to 1.00 (high sugar content). As all domains are one-dimensional, the dimension weights $w_{d}$ are always equal to 1.00 for all concepts. We assume that the dimensions are ordered like this: $d_{hue},d_{round},d_{sweet}$. The conceptual space is defined as follows in the code:
\begin{lstlisting}[language=Python]
domains = {'color':[0], 'shape':[1], 'taste':[2]}
dimension_names = ['hue', 'round', 'sweet']
space.init(3, domains, dimension_names)
\end{lstlisting}
\begin{table}[t]
  \centering
  \begin{footnotesize}
  \begin{tabular}{|l||c|c|c|c|c|c|c|c|}
    \hline
    Concept 	& $\Delta_S$& $p^-$ 				& $p^+$ 				& $\mu_0$ 	& $c$ 	& \multicolumn{3}{|c|}{$W$}\\ 
    & & & & & & $w_{\delta_{color}}$ & $w_{\delta_{shape}}$ & $w_{\delta_{taste}}$\\ \hline \hline
    Pear		& $\Delta$	& $\left(0.50, 0.40, 0.35\right)$	& $\left(0.70, 0.60, 0.45\right)$	& 1.0		& 12.0	& 0.50 & 1.25 & 1.25 \\ \hline
    Orange		& $\Delta$	& $\left(0.80, 0.90, 0.60\right)$	& $\left(0.90, 1.00, 0.70\right)$	& 1.0		& 15.0	& 1.00 & 1.00 & 1.00 \\ \hline
    Lemon		& $\Delta$	& $\left(0.70, 0.45, 0.00\right)$	& $\left(0.80, 0.55, 0.10\right)$	& 1.0		& 20.0	& 0.50 & 0.50 & 2.00 \\ \hline
    Granny & \multirow{2}{*}{$\Delta$}	& \multirow{2}{*}{$\left(0.55, 0.70, 0.35\right)$}	& \multirow{2}{*}{$\left(0.60, 0.80, 0.45\right)$}	& \multirow{2}{*}{1.0}		& \multirow{2}{*}{25.0} & \multirow{2}{*}{1.00} & \multirow{2}{*}{1.00} &  \multirow{2}{*}{1.00}\\ 
    Smith & & & & & & & &\\ \hline
    \multirow{3}{*}{Apple}	& \multirow{3}{*}{$\Delta$}	& $\left(0.50, 0.65, 0.35\right)$	& $\left(0.80, 0.80, 0.50\right)$	& \multirow{3}{*}{1.0}		& \multirow{3}{*}{10.0}	& \multirow{3}{*}{0.50} & \multirow{3}{*}{1.50} &  \multirow{3}{*}{1.00} \\ 
    			&			& $\left(0.65, 0.65, 0.40\right)$	& $\left(0.85, 0.80, 0.55\right)$	&			&		&  & & \\ 
    			&			& $\left(0.70, 0.65, 0.45\right)$	& $\left(1.00, 0.80, 0.60\right)$	&			&		&  & & \\ \hline
    \multirow{3}{*}{Banana}	& \multirow{3}{*}{$\Delta$}	& $\left(0.50, 0.10, 0.35\right)$	& $\left(0.75, 0.30, 0.55\right)$	& \multirow{3}{*}{1.0}		& \multirow{3}{*}{10.0}	& \multirow{3}{*}{0.75} & \multirow{3}{*}{1.50} &  \multirow{3}{*}{0.75} \\ 
    			&			& $\left(0.70, 0.10, 0.50\right)$	& $\left(0.80, 0.30, 0.70\right)$	&			&		&  & & \\ 
    			&			& $\left(0.75, 0.10, 0.50\right)$	& $\left(0.85, 0.30, 1.00\right)$	&			&		&  & & \\ \hline
    Red			& $\left\{\delta_{color}\right\}$ & $\left(0.90, -\infty, -\infty\right)$ & $\left(1.00, +\infty, +\infty\right)$ & 1.0 & 20.0 & 1.00 & -- & -- \\ \hline
  \end{tabular}
  \end{footnotesize}
  \caption{Definitions of several concepts.}
  \label{tab:FruitSpace}
\end{table}
\begin{figure}[tp]
\centering
\includegraphics[width=\textwidth]{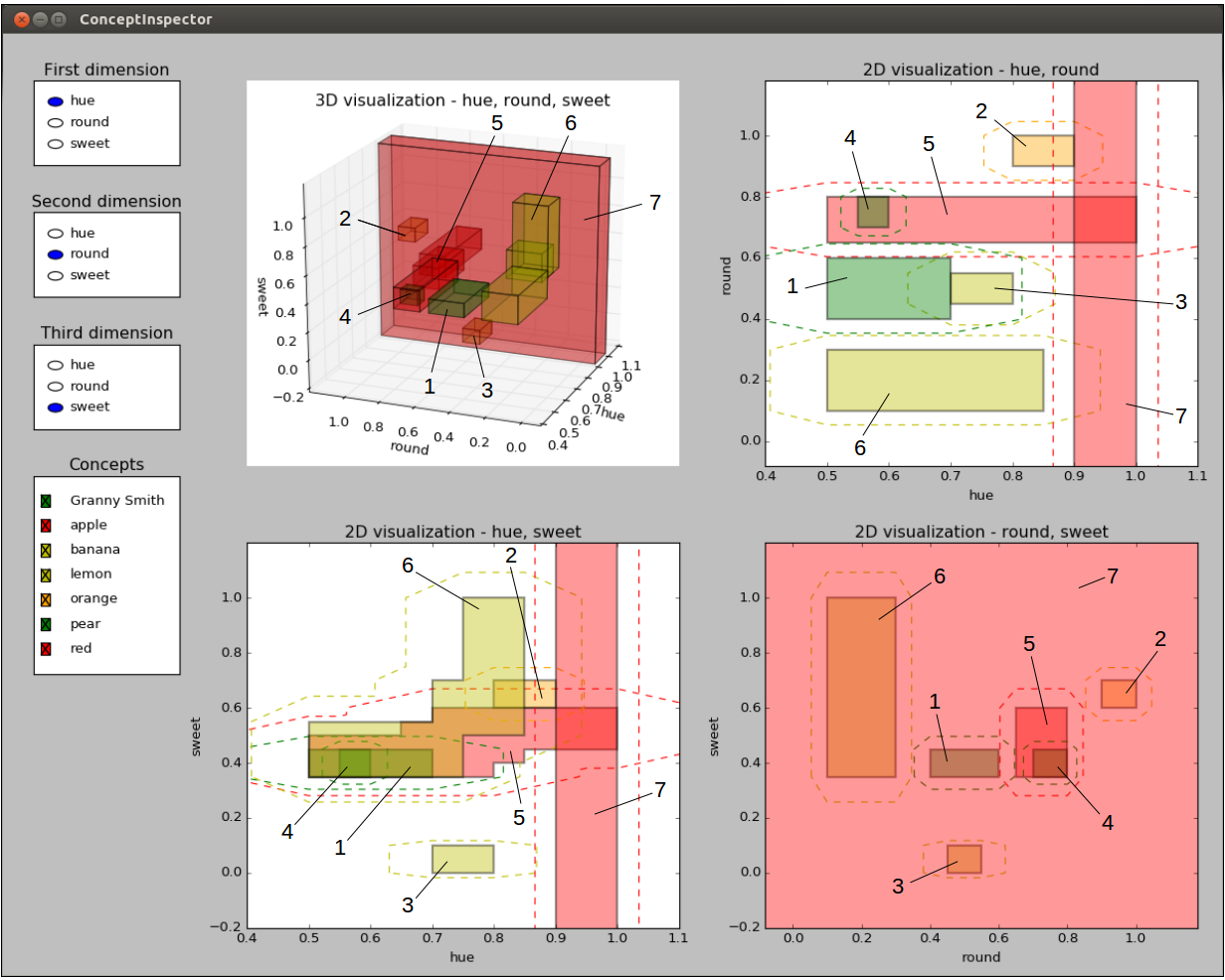}
\caption{Screenshot of the ConceptInspector tool for the fruit space example with subsequently added labels: pear (1), orange (2), lemon (3), Granny Smith (4), apple (5), banana (6), red (7). The 3D visualization only shows the concepts' cores, the 2D visualizations also illustrate the concepts' 0.5-cuts.}
\label{fig:FruitSpace}
\end{figure}
Table \ref{tab:FruitSpace} defines some concepts in this space and Figure \ref{fig:FruitSpace} visualizes them. In the code, concepts can be defined as follows:
\begin{lstlisting}[language=Python]
c_pear = Cuboid([0.5, 0.4, 0.35], [0.7, 0.6, 0.45], domains)
s_pear = Core([c_pear], domains)
w_pear = Weights({'color':0.50, 'shape':1.25, 'taste':1.25}, {'color':{0:1.0}, 'shape':{1:1.0}, 'taste':{2:1.0}})
pear = Concept(s_pear, 1.0, 12.0, w_pear)
\end{lstlisting}

We can load the definition of this fruit space into our python interpreter and apply the different operations described in Section \ref{Operations} to these concepts. This looks for example as follows:
\begin{lstlisting}[language=Python]
>>> execfile('fruit_space.py')
>>> granny_smith.subset_of(apple)
1.0
>>> apple.implies(red)
0.3333333333333332
>>> (pear.similarity_to(apple), pear.similarity_to(lemon))
(0.007635094218859955, 1.8553913626159717e-07)
>>> print apple.intersect_with(pear)
core: {[0.5, 0.625, 0.35]-[0.7, 0.625, 0.45]}
mu: 0.6872892788
c: 10.0
weights: <{'color': 0.5, 'taste': 1.125, 'shape': 1.375},
   {'color': {0: 1.0}, 'taste': {2: 1.0}, 'shape': {1: 1.0}}>
\end{lstlisting}

\section{Related Work}
\label{RelatedWork}

Our work is of course not the first attempt to devise an implementable formalization of the conceptual spaces framework. In this section, we review its strongest competitors.\\

An early and very thorough formalization was done by \cite{Aisbett2001}. Like we, they consider concepts to be regions in the overall conceptual space. However, they stick with G\"{a}rdenfors' assumption of convexity and do not define concepts in a parametric way. The only operations they provide are distance and similarity of points and regions. Their formalization targets the interplay of symbols and geometric representations, but it is too abstract to be implementable.\\

\cite{Rickard2006,Rickard2007} provide a formalization based on fuzziness. Their starting points are properties defined in individual domains. They represent concepts as co-occurence matrices of properties. For instance, $C\left(sour, green\right) = 0.8$ in the ``apple'' concept represents that in 80\% of the cases where the property ``sour'' was present, also the property ``green'' was observed. By using some mathematical transformations, Rickard et al. interpret these matrices as fuzzy sets on the universe of ordered property pairs. Operations defined on these concepts include similarity judgements between concepts and between concepts and instances. The representation of Rickard et al. nicely captures the correlations between different properties, but their representation of correlations is not geometrical: They first discretize the domains by defining properties and then compute co-occurence statistics between these properties. Depending on the discretization, this might lead to a relatively coarse-grained notion of correlation. Moreover, as properties and concepts are represented in different ways, one has to use different learning and reasoning mechanisms for them. The formalization by Rickard et al. is also not easy to work with due to the complex mathematical transformations involved.

We would also like to point out that something very similar to the co-occurence values used by Rickard et al. can be extracted from our representation. One can interpret for instance $C\left({sour},{green}\right)$ as the degree of truth of the implication $\widetilde{S}_{sour} \Rightarrow \widetilde{S}_{green}$ within the apple concept. This number can be computed by using our implication operation: 
$$C\left({sour}, {green}\right) = Impl\left(I\left(\widetilde{S}_{sour}, \widetilde{S}_{apple}\right), I\left(\widetilde{S}_{green}, \widetilde{S}_{apple}\right)\right)$$

\cite{Adams2009} represent concepts by one convex polytope per domain. This allows for efficient computations while supporting a more fine-grained representation than our cuboid-based approach. The Manhattan metric is used to combine different domains. However, correlations between different domains are not taken into account as each convex polytope is only defined on a single domain. Adams and Raubal also define operations on concepts, namely intersection, similarity computation, and concept combination. This makes their formalization quite similar in spirit to ours. 

One could generalize their approach by using polytopes that are defined on the overall space and that are convex under the Euclidean and star-shaped under the Manhattan metric. However, we have found that this requires additional constraints in order to ensure star-shapedness. The number of these constraints grows exponentially with the number of dimensions. Each modification of a concept's description would then involve a large constraint satisfaction problem, rendering this representation unsuitable for learning processes. Our cuboid-based approach is more coarse-grained, but it only involves a single constraint, namely that the intersection of the cuboids is not empty.\\

\cite{Lewis2016} formalize conceptual spaces using random set theory. A random set can be characterized by a set of prototypical points $P$ and a threshold $\epsilon$. Instances that have a distance of at most $\epsilon$ to the prototypical set are considered to be elements of the set. The threshold is however not exactly determined, only its probability distribution $\delta$ is known. Based on this uncertainty, a membership function $\mu\left(x\right)$ can be defined that corresponds to the probability of the distance of a given point $x$ to any prototype $p \in P$ being smaller than $\epsilon$. Lewis and Lawry define properties as random sets within single domains and concepts as random sets in a boolean space whose dimensions indicate the presence or absence of properties. In order to define this boolean space, a single property is taken from each domain. This is in some respect similar to the approach of \cite{Rickard2006, Rickard2007} where concepts are also defined on top of existing properties. However, whereas Rickard et al. use two separate formalisms for properties and concepts, Lewis and Lawry use random sets for both (only the underlying space differs). 

Lewis and Lawry illustrate how their mathematical formalization is capable of reproducing some effects from the psychological concept combination literature. However, they do not develop a way of representing correlations between domains (such as ``red apples are sweet and green apples are sour''). One possible way to do this within their framework would be to define two separate concepts ``red apple'' and ``green apple'' and then define on top of them a disjunctive concept ``apple = red apple or green apple''. This however is a quite indirect way of defining correlations, whereas our approach is intuitively much easier to grasp. Nevertheless, their approach is similar to ours in using a distance-based membership function to a set of prototypical points while using the same representational mechanisms for both properties and concepts.\\

None of the approaches listed above provides a set of operations that is as comprehensive as the one offered by our proposed formalization.
Many practical applications of conceptual spaces (e.g., \cite{Chella2003,Raubal2004,Dietze2008,Derrac2015}) use only partial ad-hoc implementations of the conceptual spaces framework which usually ignore some important aspects of the framework (e.g., the domain structure).\\

The only publicly avaliable implementation of the conceptual spaces framework that we are currently aware of is provided by \cite{Lieto2015, Lieto2017}. They propose a hybrid architecture that represents concepts by using both description logics and conceptual spaces. This way, symbolic ontological information and similarity-based ``common sense'' knowledge can be used in an integrated way. Each concept is represented by a single prototypical point and a number of exemplar points. Correlations between domains can therefore only be encoded through the selection of appropriate exemplars. Their work focuses on classification tasks and does therefore not provide any operations for combining different concepts. With respect to the larger number of supported operations, our formalization and implementation can thus be considered more general than theirs. In contrast to our work, the current implementation of their system\footnote{See \url{http://www.dualpeccs.di.unito.it/download.html}.} comes without any publicly available source code\footnote{The sorce code of an earlier and more limited version of their system can be found here: \url{http://www.di.unito.it/~lieto/cc_classifier.html}.}.
\section{Outlook and Future Work}
\label{Outlook}
As stated earlier, our overall research goal is to devise a symbol grounding mechanism by defining a concept formation process in the conceptual spaces framework.
Concept formation \cite{Gennari1989} is the process of incrementally creating a meaningful hierarchical categorization of unlabeled observations. One can easily see that a successful concept formation process implicitly solves the symbol grounding problem: If we are able to find a bottom-up process that can group observations into meaningful categories, these categories can be linked to abstract symbols. These symbols are then grounded in reality, as the concepts they refer to are generalizations of actual observations.\\

\begin{figure}[tp]
\centering
\includegraphics[width=0.82\textwidth]{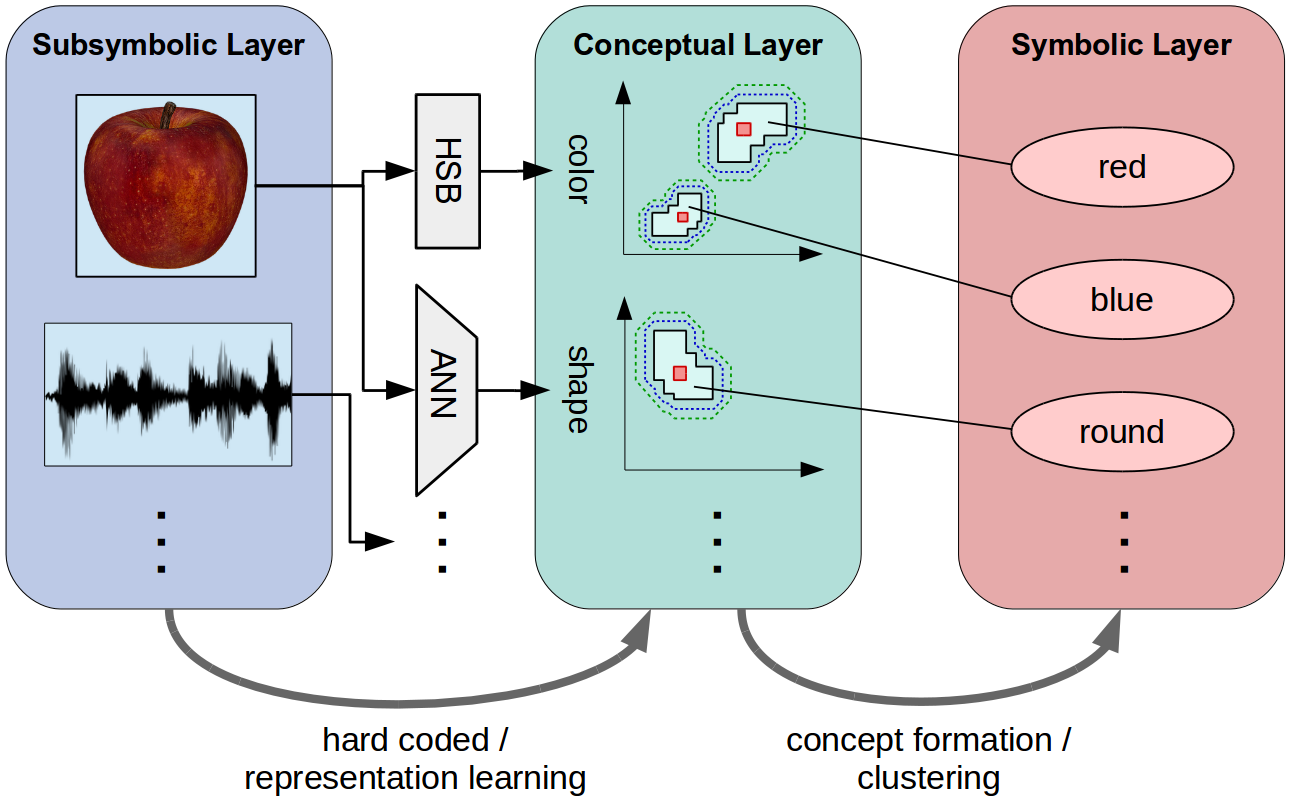}
\caption{Visualization of our envisioned symbol grounding architecture. The translation from the subsymbolic to the conceptual layer will be based on both hard-coded conversions (like the HSB color space) and pre-trained artificial neural networks. The translation from the conceptual layer to the symbolic layer will be based on a clustering algorithm that performs concept formation.}
\label{fig:Architecture}
\end{figure}

We aim for a three-layered architecture as depicted in Figure \ref{fig:Architecture} with the formalization presented in this paper serving as middle layer.
The conceptual space that we will use for our work on concept formation will have a predefined structure: Domains having a well-known structure will be hand-crafted, e.g., color or sound. These domains can quite straightforwardly be represented by a handful of dimensions. Domains with an unclear internal structure (i.e., where it is inherently hard to hand-craft a dimensional representation) can potentially be obtained by deep representation learning \cite{Bengio2013}. This includes for instance the domain of shapes. We will investigate different neural network architectures that have shown promising results in extracting meaningful dimensions from unlabeled data sets, e.g., InfoGAN \cite{Chen2016} and beta-VAE \cite{Higgins2017}.

This set-up of the conceptual space is not seen as an active part of the concept formation process, but as a preprocessing step to lift information from the subsymbolic layer (e.g., raw pixel information from images) to the conceptual layer.\\

The most straightforward way of implementing concept formation within the conceptual spaces framework is to use a clustering algorithm that groups unlabled instances into meaningful regions. Each of these regions can be interpreted as a concept and a symbol can be attached to it.

As this research is seen in the context of artificial general intelligence, it will follow the assumption of insufficient knowledge and resources \cite{Wang2011}: The system will only have \emph{limited resources}, i.e., it will not be able to store all observed data points. This calls for an incremental clustering process. The system will also have to cope with \emph{incomplete information}, i.e., with incomplete feature vectors. 
As concept hierarchies are an important and useful aspect of human conceptualizations \cite[Chapter 7]{Murphy2002}, the concept formation process should ideally result in a hierarchy of concepts.
It is generally unknown in the beginning how many concepts will be discovered. Therefore, the number of concepts must be adapted over time.

There are some existing clustering algorithms that partially fulfill these requirements (e.g., CLASSIT \cite{Gennari1989} or SUSTAIN \cite{Love2004}), but none of them is a perfect fit. We will take inspiration from these algorithms in order to devise a new algorithm that fulfills all the requirements stated above. One can easily see that our formalization is able to support such clustering processes: Concepts can be created and deleted. Modifying the support points of the cuboids in a concept's core results in changes to the concept's position, size, and form. One must however ensure that such modifications preserve the non-emptiness of the cuboids' intersection. Moreover, a concept's form can be changed by modifying the parameters $c$ and $W$: By changing $c$, one can control the overall degree of fuzziness, and by changing $W$, one can control how this fuzziness is distributed among the different domains and dimensions.
Two neighboring concepts $\widetilde{S}_1, \widetilde{S}_2$ can be merged into a single cluster by unifying them. A single concept can be split up into two parts by using the axis-parallel cut operation.\\

Finally, our formalization also supports reasoning processes: 
\cite{Gardenfors2000} argues that adjective-noun combinations like ``green apple'' or ``purple banana'' can be expressed by combining properties with concepts. This is supported by our operations of intersection and subspace projection:

In combinations like ``green apple'', property and concept are compatible. We expect that their cores intersect and that the $\mu_0$ parameter of their intersection is therefore relatively large. In this case, ``green'' should narrow down the color information associated with the ``apple'' concept. This can be achieved by simply computing their intersection.

In combinations like ``purple banana'', property and concept are incompatible. We expect that their cores do not intersect and that the $\mu_0$ parameter of their intersection is relatively small. In this case, ``purple'' should replace the color information associated with the ``banana'' concept. This can be achieved by first removing the color domain from the ``banana'' concept (through a subspace projection) and by then intersecting this intermediate result with ``purple''.

\section{Conclusion}
\label{Conclusion}

In this paper, we proposed a new formalization of the conceptual spaces framework. We aimed to geometrically represent correlations between domains, which led us to consider the more general notion of star-shapedness instead of G\"{a}rden\-fors' favored constraint of convexity. We defined concepts as fuzzy sets based on intersecting cuboids and a similarity-based membership function. Moreover, we provided a comprehensive set of operations, both for creating new concepts based on existing ones and for measuring relations between concepts. This rich set of operations makes our formalization (to the best of our knowledge) the most thorough and comprehensive formalization of conceptual spaces developed so far.

Our implementation of this formalization and its source code are publicly avaliable and can be used by any researcher interested in conceptual spaces. We think that our implementation can be a good foundation for practical research on conceptual spaces and that it will considerably facilitate research in this area.

In future work, we will provide more thorough definitions of similarity and betweenness for concepts, given that our current definitions are rather naive. A potential starting point for this can be the betwenness relations defined by \cite{Derrac2015}. Moreover, we will use the formalization proposed in this paper as a starting point for our research on concept formation as outlined in Section \ref{Outlook}.

\bibliography{/home/lbechberger/Documents/Papers/jabref.bib}{}
\bibliographystyle{apalike}

\end{document}